\documentclass[11pt]{article}

\usepackage{fullpage,times}
\usepackage{url}

\usepackage{amsthm,amsfonts,amsmath,amssymb,epsfig,color,float,graphicx,verbatim}
\usepackage{algorithm,algorithmic}

\newtheorem{theorem}{Theorem}

\newtheorem{lemma}{Lemma}
\newtheorem{corollary}{Corollary}

\newcommand{\reals}{\mathbb{R}}
\newcommand{\E}{\mathbb{E}}

\newcommand{\bw}{\mathbf{w}}
\newcommand{\bg}{\mathbf{g}}

\newcommand{\bu}{\mathbf{u}}

\newcommand{\bn}{\mathbf{n}}

\newcommand{\btheta}{\boldsymbol{\theta}}

\newcommand{\Ncal}{\mathcal{N}}

\newcommand{\Wcal}{\mathcal{W}}

\newcommand{\norm}[1]{\|#1\|}
\newcommand{\inner}[1]{\langle#1\rangle}

\renewcommand{\eqref}[1]{Eq.~(\ref{#1})}
\newcommand{\lemref}[1]{Lemma~\ref{#1}}

\newcommand{\thmref}[1]{Thm.~\ref{#1}}

\title{An Optimal Algorithm for Bandit and Zero-Order Convex Optimization with Two-Point Feedback}
\author{Ohad Shamir\\Weizmann Institute of Science\\\texttt{ohad.shamir@weizmann.ac.il}}
\date{}

\begin{document}

\maketitle

\begin{abstract}
  We consider the closely related problems of bandit convex optimization
  with two-point feedback, and zero-order stochastic convex optimization
  with two function evaluations per round. We provide a simple algorithm
  and analysis which is optimal for convex Lipschitz functions. This
  improves on \cite{dujww13}, which only provides an optimal result for
  smooth functions; Moreover, the algorithm and analysis are simpler, and
  readily extend to non-Euclidean problems. The algorithm is based on a
  small but surprisingly powerful modification of the gradient estimator.
\end{abstract}

\section{Introduction}

We consider the problem of bandit convex optimization with two-point feedback
\cite{AgDeXi10}. This problem can be defined as a repeated game between a
learner and an adversary as follows: At each round $t$, the adversary picks a
convex function $f_t$ on $\reals^d$, which is not revealed to the learner.
The learner then chooses a point $\bw_t$ from some known and closed convex
set $\Wcal\subseteq \reals^d$, and suffers a loss $f_t(\bw_t)$. As feedback,
the learner may choose two points $\bw'_t,\bw''_t\in \Wcal$ and
receive\footnote{This is slightly different than the model of
\cite{AgDeXi10}, where the learner only chooses $\bw'_t,\bw''_t$ and the loss
is $\frac{1}{2}\left(f_t(\bw'_t)+f_t(\bw''_t)\right)$. However, our results
and analysis can be easily translated to their setting, and the model we
discuss translates more directly to the zero-order stochastic optimization
considered later.} $f_t(\bw_t'),f_t(\bw_t'')$. The learner's goal is to
minimize average regret, defined as
\[
\frac{1}{T}\sum_{t=1}^{T}f_t(\bw_t)-\min_{\bw\in\Wcal}\frac{1}{T}\sum_{t=1}^{T}f_t(\bw).
\]
In this note, we focus on obtaining bounds on the expected average regret
(with respect to the learner's randomness).

A closely-related and easier setting is zero-order stochastic convex
optimization. In this setting, our goal is to approximately solve
$F(\bw)=\min_{\bw\in\Wcal}\E_{\xi}[f(\bw;\xi)]$, given limited access to
$\{f(\cdot;\xi_t)\}_{t=1}^{T}$ where $\xi_t$ are i.i.d. instantiations.
Specifically, we assume that each $f(\cdot,\xi_t)$ is not directly observed,
but rather can be queried at two points. This models situations where
computing gradients directly is complicated or infeasible. It is well-known
\cite{cesa2004generalization} that given an algorithm with expected average
regret $R_T$ in the bandit optimization setting above, if we feed it with the
functions $f_t(\bw)=f(\bw;\xi_t)$, then the average
$\bar{\bw}_T=\frac{1}{T}\sum_{t=1}^{T}\bw_t$ of the points generated
satisfies the following bound on the expected optimization error:
\[
\E[F(\bar{\bw}_T)]-\min_{\bw\in\Wcal}F(\bw)\leq R_T.
\]
Thus, an algorithm for bandit optimization can be converted to an algorithm
for zero-order stochastic optimization with similar guarantees.

The bandit optimization setting with two-point feedback was proposed and
studied in \cite{AgDeXi10}. Independently, \cite{nesterov2011random} and
considered two-point methods for stochastic optimization. Both papers are
based on randomized gradient estimates which are then fed into standard
first-order algorithms (e.g. gradient descent, or more generally mirror
descent). However, the regret/error guarantees in both papers were suboptimal
in terms of the dependence on the dimension. Recently, \cite{dujww13}
considered a similar approach for the stochastic optimization setting,
attaining an optimal error guarantee when $f(\cdot;\xi)$ is a smooth function
(differential and with Lipschitz-continuous gradients). Related results in
the smooth case were also obtained by \cite{GL13}. However, to tackle the
general case, where $f(\cdot;\xi)$ may be non-smooth, \cite{dujww13} resorted
to a non-trivial smoothing scheme and a significantly more involved analysis.
The resulting bounds have additional factors (logarithmic in the dimension)
compared to the guarantees in the smooth case. Moreover, an analysis is only
provided for Euclidean problems (where the domain $\Wcal$ and Lipschitz
parameter of $f_t$ scale with the $L_2$ norm).

In this note, we present and analyze a simple algorithm with the following
properties:
\begin{itemize}
  \item For Euclidean problems, it is optimal up to constants for both
      smooth and non-smooth functions. This closes the gap between the
      smooth and non-smooth Euclidean problems in this setting.
  \item The algorithm and analysis are readily applicable to non-Euclidean
      problems. We give an example for the $1$-norm, with the resulting
      bound optimal up to a $\sqrt{\log(d)}$ factor.
  \item The algorithm and analysis are simpler than those proposed in
      \cite{dujww13}. They apply equally to the bandit and zero-order
      optimization setting, and can be readily extended using standard
      techniques (e.g. to strongly-convex functions, regret/error bounds
      holding with high-probability rather than just in expectation, and
      improved bounds if allowed $k>2$ observations per round instead of
      just two).
\end{itemize}

Like previous algorithms, our algorithm is based on a random gradient
estimator, which given a function $f$ and point $\bw$, queries $f$ at two
random locations close to $\bw$, and computes a random vector whose
expectation is a gradient of a smoothed version of $f$. The papers
\cite{nesterov2011random,dujww13,GL13} essentially use the estimator which
queries at $\bw$ and $\bw+\delta\bu$ (where $\bu$ is a random unit vector and
$\delta>0$ is a small parameter), and returns
\begin{equation}\label{eq:they}
\frac{d}{\delta}\left(f(\bw+\delta\bu)-f(\bw)\right)\bu.
\end{equation}
The intuition is readily seen in the one-dimensional ($d=1$) case, where the
expectation of this expression equals
\begin{equation}\label{eq:they2}
\frac{1}{2\delta}\left(f(w+\delta)-f(w-\delta)\right),
\end{equation}
which indeed approximates the derivative of $f$ (assuming $f$ is
differentiable) at $w$, if $\delta$ is small enough.

In contrast, our algorithm uses a slightly different estimator (also used in
\cite{AgDeXi10}), which queries at $\bw-\delta\bu, \bw+\delta\bu$, and
returns
\begin{equation}\label{eq:we}
\frac{d}{2\delta}\left(f(\bw+\delta\bu)-f(\bw-\delta\bu)\right)\bu.
\end{equation}
Again, the intuition is readily seen in the case $d=1$, where the expectation
of this expression also equals \eqref{eq:they2}.

When $\delta$ is sufficiently small and $f$ is differentiable at $\bw$, both
estimators compute a good approximation of the true gradient $\nabla f(\bw)$.
However, when $f$ is not differentiable, the variance of the estimator in
\eqref{eq:they} can be quadratic in the dimension $d$, as pointed out by
\cite{dujww13}: For example, for $f(\bw)=\norm{\bw}_2$ and $\bw=0$, the
second moment equals
\[
\E\left[\left\|\frac{d}{\delta}\left(f(\delta\bu)-f(\mathbf{0})\right)\bu\right\|^2\right]
~=~ \E\left[d^2\norm{\bu}^2\right]~=~ d^2.
\]
Since the performance of the algorithm crucially depends on the second moment
of the gradient estimate, this leads to a highly sub-optimal guarantee. In
\cite{dujww13}, this was handled by adding an additional random perturbation
and using a more involved analysis. Surprisingly, it turns out that the
slightly different estimator in \eqref{eq:we} does not suffer from this
problem, and its second moment is essentially linear in the dimension $d$.

\section{Algorithm and Main Results}

We consider the algorithm described in Figure \ref{alg:quadratic}, which
performs standard mirror descent using a randomized gradient estimator
$\tilde{\bg}_t$ of a (smoothed) version of $f_t$ at point $\bw_t$. We make
the assumption that one can indeed query $f_t$ at any point
$\bw_t+\delta_t\bu_t$ as specified in the algorithm\footnote{This may require
us to query at a distance $\delta_t$ outside $\Wcal$. If we must query within
$\Wcal$, then one can simply run the algorithm on a slightly smaller set
$(1-\delta)\Wcal$, where $\delta\geq \delta_t$ for all $t$, ensuring that we
always query at $\Wcal$. Since the formal guarantee in \thmref{thm:main}
holds for arbitrarily small $\delta_t$, and each $f_t$ is Lipschitz, we can
always take $\delta$ and $\delta_t$ small enough so that the additional
regret/error incurred is negligible.}.

\begin{algorithm}
\caption{Two-Point Bandit Convex Optimization Algorithm}
\label{alg:quadratic}
\begin{algorithmic}
    \STATE Input: Step size $\eta$, function $r:\Wcal\mapsto \reals$, exploration parameters $\delta_t>0$
    \STATE Initialize $\btheta_1=\mathbf{0}$.
    \FOR{$t=1,\ldots,T-1$}
    \STATE Predict $\bw_t=\arg\max_{\bw\in\Wcal}\inner{\btheta_t,\bw}-r(\bw)$
    \STATE Sample $\bu_t$ uniformly from the Euclidean unit sphere $\{\bw:\norm{\bw}_2=1\}$
    \STATE Query $f_t(\bw_t+\delta_t \bu_t)$ and $f_t(\bw_t-\delta_t \bu_t)$
    \STATE Set $\tilde{\bg}_t=\frac{d}{2\delta_t}\left(f_t(\bw_t+\delta_t\bu_t)-f_t(\bw_t-\delta_t\bu_t)\right)\bu_t$
    \STATE Update $\btheta_{t+1} = \btheta_t-\eta\tilde{\bg_t}$
    \ENDFOR
\end{algorithmic}
\end{algorithm}

The analysis of the algorithm is presented in the following theorem:

\begin{theorem}\label{thm:main}
Assume the following conditions hold:
\begin{enumerate}
  \item\label{it:1} $r$ is $1$-strongly convex with respect to a norm
      $\norm{\cdot}$, and $\sup_{\bw\in\Wcal}r(\bw)\leq R^2$ for some
      $R<\infty$.
  \item\label{it:2} $f_t$ is convex and $G_2$-Lipschitz with respect to the
      $2$-norm $\norm{\cdot}_2$.
  \item\label{it:3} The dual norm $\norm{\cdot}_*$ of $\norm{\cdot}$ is
      such that $\sqrt[4]{\E_{\bu_t}\norm{\bu_t}_*^4}\leq p_*$ for some
      $p_{*}<\infty$.
\end{enumerate}
  If $\eta = \frac{R}{p_{*} G_2 \sqrt{dT}}$, and $\delta_t$ chosen such that
  $\delta_t\leq p_*R\sqrt{\frac{d}{T}}$, then the sequence $\bw_1,\ldots,\bw_T$
  generated by the algorithm satisfies the following for any $T$ and
  $\bw^*\in\Wcal$:
  \[
  \E\left[\frac{1}{T}\sum_{t=1}^{T}f_t(\bw_t)-\frac{1}{T}\sum_{t=1}^{T}f_t(\bw^*)\right]
  \leq c~p_{*} G_2R\sqrt{\frac{d}{T}},
  \]
  where $c$ is some numerical constant.
\end{theorem}

We note that conditions \ref{it:1} is standard in the analysis of the
mirror-descent method (see the specific corollaries below), whereas
conditions \ref{it:2} and \ref{it:3} are needed to ensure that the variance
of our gradient estimator is controlled.

As mentioned earlier, the bound on the average regret which appears in
\thmref{thm:main} immediately implies a similar bound on the error in a
stochastic optimization setting, for the average point
$\bar{\bw}_T=\frac{1}{T}\sum_{t=1}^{T}\bw_t$. We note that the result is
robust to the choice of $\eta$, and is the same up to constants as long as
$\eta=\Theta(R/p_* G_2\sqrt{dT})$. Also, the constant $c$, while always
bounded above zero, shrinks as $\delta_t\rightarrow 0$ (see the proof for
details).

As a first application, let us consider the case where $\norm{\cdot}$ is the
Euclidean norm $\norm{\cdot}_2$. In this case, we can take
$s(\bw)=\frac{1}{2}\norm{\bw}_2^2$, and the algorithm reduces to a standard
variant of online gradient descent, defined as
$\btheta_{t+1}=\btheta_t-\tilde{\bg}_t$ and
$\bw_t=\arg\min_{\bw\in\Wcal}\norm{w-\btheta_t}_2$. In this case, we get the
following corollary:
\begin{corollary}\label{cor:2}
Suppose $f_t$ for all $t$ is $G_2$-Lipschitz with respect to the Euclidean
norm, and $\Wcal\subseteq \{\bw:\norm{\bw}_2\leq R\}$. Then using
$\norm{\cdot}=\norm{\cdot}_2$ and $r(\bw)=\frac{1}{2}\norm{\bw}_2^2$, it
holds for some constant $c$ and any $\bw^*\in\Wcal$ that
\[
  \E\left[\frac{1}{T}\sum_{t=1}^{T}f_t(\bw_t)-\frac{1}{T}\sum_{t=1}^{T}f_t(\bw^*)\right]
  \leq c~G_2 R\sqrt{\frac{d}{T}},
\]
\end{corollary}
The proof is immediately obtained from \thmref{thm:main}, noting that $p_*=1$
in our case. This bound matches (up to constants) the lower bound in
\cite{dujww13}, hence closing the gap between upper and lower bounds in this
setting.

As a second application, let us consider the case where $\norm{\cdot}$ is the
$1$-norm, $\norm{\cdot}_1$, the domain $\Wcal$ is the simplex in $\reals^d$,
$d>1$ (although our result easily extends to any subset of the $1$-norm unit
ball), and we use a standard entropic regularizer:
\begin{corollary}\label{cor:1}
Suppose $f_t$ for all $t$ is $G_1$-Lipschitz with respect to the $L_1$ norm.
Then using $\norm{\cdot}=\norm{\cdot}_1$ and $r(\bw)=\sum_{i=1}^{d}w_i\log(d
w_i)$, it holds for some constant $c$ and any $\bw^*\in\Wcal$ that
\[
 \E\left[\frac{1}{T}\sum_{t=1}^{T}f_t(\bw_t)-\frac{1}{T}\sum_{t=1}^{T}f_t(\bw^*)\right]
  \leq c~G_1 \sqrt{\frac{d\log^2(d)}{T}}.
\]
\end{corollary}
This bound matches (this time up to a logarithmic factor) the lower bound in
\cite{dujww13} for this setting .
\begin{proof}
The function $r$ is $1$-strongly convex with respect to the $1$-norm (see for
instance \cite{SSS12}, Example 2.5), and has value at most $\log(d)$ on the
simplex. Also, if $f_t$ is $G_1$-Lipschitz with respect to the $1$-norm, then
it must be $\sqrt{d}G_1$-Lipschitz with respect to the Euclidean norm.
Finally, to satisfy condition \ref{it:3} in \thmref{thm:main}, we upper bound
$\sqrt[4]{\E[\norm{\bu_t}_\infty^4]}$ using the following lemma, whose proof
is given in the appendix:
\begin{lemma}\label{lem:uinfbound}
  If $\bu$ is uniformly distributed on the unit sphere in $\reals^d$, $d>1$, then
  $\sqrt[4]{\E[\norm{\bu}_\infty^4]}\leq c\sqrt{\frac{\log(d)}{d}}$ where
  $c$ is a positive numerical constant independent of $d$.
\end{lemma}
Plugging these observations into \thmref{thm:main} leads to the desired
result.
\end{proof}

\section{Proof of Theorem \ref{thm:main}}

As discussed in the introduction, the key to getting improved results
compared to previous papers is the use of a slightly different random
gradient estimator, which turns out to have significantly less variance. The
formal proof relies on a few simple lemmas listed below. The key lemma is
\lemref{lem:moment}, which establishes the improved variance behavior.

\begin{lemma}\label{lem:zink}
  For any $\bw^*\in\Wcal$, it holds that
  \[
  \sum_{t=1}^{T}\inner{\tilde{\bg}_t,\bw_t-\bw^*}\leq \frac{1}{\eta}R^2+\eta\sum_{t=1}^{T}\norm{\tilde{\bg}_t}_*^2.
  \]
\end{lemma}
This lemma is the canonical result on the convergence of online mirror
descent, and the proof is standard (see e.g. \cite{SSS12}).

\begin{lemma}\label{lem:hatell}
Define the function
\[
\hat{f}_t(\bw) = \E_{\bu_t}\left[f_t(\bw+\delta_t \bu_t)\right],
\]
over $\Wcal$, where $\bu_t$ is a vector picked uniformly at random from the
Euclidean unit sphere. Then the function is convex, Lipschitz with constant
$G_2$, satisfies
\[
\sup_{\bw\in\Wcal}|\hat{f}_t(\bw)-f_t(\bw)|\leq \delta_t G_2,
\]
and is differentiable with the following gradient:
\[
\nabla \hat{f}_t(\bw) =\E_{\bu_t}\left[\frac{d}{\delta_t}f_t(\bw+\delta_t\bu_t)\bu_t\right].
\]
\end{lemma}
\begin{proof}
  The fact that the function is convex and Lipschitz is immediate from its
  definition and the assumptions in the theorem. The inequality follows from $\bu_t$ being a unit vector and that $f_t$ is assumed to be
  $G_2$-Lipschitz with respect to the $2$-norm. The differentiability property follows from
  Lemma 2.1 in \cite{FlaxKaMc05}.
\end{proof}

\begin{lemma}\label{lem:var}
For any function $g$ which is $L$-Lipschitz with respect to the $2$-norm, it
holds that if $\bu$ is uniformly distributed on the Euclidean unit sphere,
then
\[
\sqrt{\E\left[\left(g(\bu)-\E[g(\bu)]\right)^4\right]} \leq c\frac{L^2}{d}.
\]
for some numerical constant $c$.
\end{lemma}
\begin{proof}
A standard result on the concentration of Lipschitz functions on the
Euclidean unit sphere implies that
\[
\Pr(|g(\bu)-\E[g(\bu)]|>t) \leq 2\exp\left(-c'd t^2/L^2\right)
\]
for some numerical constant $c'>0$ (see the proof of Proposition 2.10 and
Corollary 2.6 in \cite{ledoux2005concentration}). Therefore,
\begin{align*}
&\sqrt{\E\left[\left(g(\bu)-\E[g(\bu)]\right)^4\right]}
= \sqrt{\int_{t=0}^{\infty}\Pr\left(\left(g(\bu)-\E[g(\bu)]\right)^4>t\right)dt}\\
&= \sqrt{\int_{t=0}^{\infty}\Pr\left(\left|g(\bu)-\E[g(\bu)]\right|>\sqrt[4]{t}\right)dt}
\leq \sqrt{\int_{t=0}^{\infty}2\exp\left(-\frac{c'd\sqrt{t}}{L^2}\right)dt}
= \sqrt{2\frac{L^4}{(c'd)^2}},
\end{align*}
which equals $cL^2/d$ for some numerical constant $c$.
\end{proof}

\begin{lemma}\label{lem:moment}
  It holds that $\E[\tilde{\bg}_t|\bw_t] = \nabla \hat{f}_t(\bw_t)$ (where $\hat{f}_t(\cdot)$ is as defined in \lemref{lem:hatell}), and
  $\E[\norm{\tilde{\bg}_t}^2|\bw_t] \leq cd p_{*}^2 G_2^2$ for some numerical constant
  $c$.
\end{lemma}
\begin{proof}
  For simplicity of notation, we drop the $t$ subscript. Since $\bu$ has a symmetric distribution around the origin,
  \begin{align*}
  \E[\tilde{\bg}|\bw] &= \E_{\bu}\left[\frac{d}{2\delta}\left(f(\bw+\delta\bu)-f(\bw-\delta\bu)\right)\bu\right]\\
  &=\E_{\bu}\left[\frac{d}{2\delta}\left(f(\bw+\delta\bu)\right)\bu\right]+\E_{\bu}\left[\frac{d}{2\delta}f(\bw-\delta\bu)(-\bu)\right]\\
  &=\E_{\bu}\left[\frac{d}{2\delta}\left(f(\bw+\delta\bu)\right)\bu\right]+\E_{\bu}\left[\frac{d}{2\delta}f(\bw+\delta\bu)(\bu)\right]\\
  &=\E_{\bu}\left[\frac{d}{\delta}f(\bw+\delta\bu)\bu\right]
  \end{align*}
  which equals $\nabla \hat{f}(\bw)$ by \lemref{lem:hatell}.

  As to the second part of the lemma, we have the following, where $\alpha$ is an arbitrary parameter and where we use the elementary inequality $(a-b)^2\leq 2(a^2+b^2)$.
  \begin{align*}
  \E[\norm{\tilde{\bg}}_*^2|\bw] &= \E_{\bu}\left[\norm{\frac{d}{2\delta}\left(f(\bw+\delta\bu)-f(\bw-\delta\bu)\right)\bu}_*^2\right]\\
  &= \frac{d^2}{4\delta^2}\E_{\bu}\left[\norm{\bu}_*^2 \left(f(\bw+\delta\bu)-f(\bw-\delta\bu)\right)^2\right]\\
  &= \frac{d^2}{4\delta^2}\E_{\bu}\left[\norm{\bu}_*^2 \left(\left(f(\bw+\delta\bu)-\alpha\right)-\left(f(\bw-\delta\bu)-\alpha\right)\right)^2\right]\\
  &\leq \frac{d^2}{2\delta^2}\E_{\bu}\left[\norm{\bu}_*^2 \left(\left(f(\bw+\delta\bu)-\alpha\right)^2+\left(f(\bw-\delta\bu)-\alpha\right)^2\right)\right]\\
  &= \frac{d^2}{2\delta^2}\left(\E_{\bu}\left[\norm{\bu}_*^2 \left(f(\bw+\delta\bu)-\alpha\right)^2\right]+\E_{\bu}\left[\norm{\bu}_*^2\left(f(\bw-\delta\bu)-\alpha\right)^2\right]\right).
  \end{align*}
  Again using the symmetrical distribution of $\bu$, this equals
  \begin{align*}
      \frac{d^2}{2\delta^2}&\left(\E_{\bu}\left[\norm{\bu}_*^2 \left(f(\bw+\delta\bu)-\alpha\right)^2\right]+\E_{\bu}\left[\norm{\bu}_*^2\left(f(\bw+\delta\bu)-\alpha\right)^2\right]\right)\\
      &\frac{d^2}{\delta^2}\E_{\bu}\left[\norm{\bu}_*^2\left(f(\bw+\delta\bu)-\alpha\right)^2\right].
  \end{align*}
  Applying Cauchy-Schwartz and using the condition $\sqrt[4]{\E_{\bu}\norm{\bu}_*^4}\leq p_*$ stated in the theorem, we get the upper bound
  \[
  \frac{d^2}{\delta^2}\sqrt{\E_{\bu}\left[\norm{\bu}_*^4\right]}\sqrt{\E_{\bu}\left[\left(f(\bw+\delta\bu)
  -\alpha\right)^4\right]} ~=~ \frac{p_{*}^2 d^2}{\delta^2}\sqrt{\E_{\bu}\left[\left(f(\bw+\delta\bu)
  -\alpha\right)^4\right]}.
  \]
  In particular, taking $\alpha=\E_{\bu}[f(\bw+\delta\bu)]$
  and using \lemref{lem:var} (noting that $f(\bw+\delta\bu)$ is $G_2\delta$-Lipschitz
  w.r.t. $\bu$ in terms of the $2$-norm), this is at most
  $\frac{p_{*}^2 d^2}{\delta^2}c\frac{(G_2\delta)^2}{d} = cd p_{*}^2 G_2^2$ as
  required.
\end{proof}

We are now ready to prove the theorem. Taking expectations on both sides of
the inequality in \lemref{lem:zink}, we have
\begin{equation}\label{eq:ineqexp}
\E\left[\sum_{t=1}^{T}\inner{\tilde{\bg}_t,\bw_t-\bw^*}\right]~\leq~ \frac{1}{\eta}R^2+\eta\sum_{t=1}^{T}\E\left[\norm{\tilde{\bg}_t}_*^2\right]
~=~ \frac{1}{\eta}R^2+\eta\sum_{t=1}^{T}\E\left[\E\left[\norm{\tilde{\bg}_t}_*^2|\bw_t\right]\right].
\end{equation}
Using \lemref{lem:moment}, the right hand side is at most
\[
\frac{1}{\eta}R^2+\eta cd p_{*}^2 G_2^2 T
\]
The left hand side of \eqref{eq:ineqexp}, by \lemref{lem:moment} and
convexity of $\hat{f}_t$, equals
\[
\E\left[\sum_{t=1}^{T}\inner{\E[\tilde{\bg}_t|\bw_t],\bw_t-\bw^*}\right]
= \E\left[\sum_{t=1}^{T}\inner{\nabla \hat{f}_t(\bw_t),\bw_t-\bw^*}\right]
\geq \E\left[\sum_{t=1}^{T}\left(\hat{f}_t(\bw_t)-\hat{f}_t(\bw^*)\right)\right].
\]
By \lemref{lem:hatell}, this is at least
\[
\E\left[\sum_{t=1}^{T}\left(f_t(\bw_t)-f_t(\bw^*)\right)\right]-G_2\sum_{t=1}^{T}\delta_t.
\]
Combining these inequalities and plugging back into \eqref{eq:ineqexp}, we
get
\[
\E\left[\sum_{t=1}^{T}\left(f_t(\bw_t)-f_t(\bw^*)\right)\right] \leq G_2\sum_{t=1}^{T}\delta_t
+\frac{1}{\eta}R^2+cd p_{*}^2 G_2^2\eta T.
\]
Choosing $\eta=R/(p_* G_2 \sqrt{dT})$, and any $\delta_t\leq p_*R\sqrt{d/T}$,
we get
\[
\E\left[\sum_{t=1}^{T}\left(f_t(\bw_t)-f_t(\bw^*)\right)\right] \leq (c+2)p_* G_2 R\sqrt{dT}.
\]
Dividing both sides by $T$, the result follows.

\bibliographystyle{plain}
\bibliography{mybib}

\appendix

\section{Proof of \lemref{lem:uinfbound}}

We note that the distribution of $\bu$ is equivalent to that of
$\frac{\norm{\bn}_{\infty}^4}{\norm{\bn}_2^4}$, where
$\bn\sim\Ncal\left(\mathbf{0},I_d\right)$ is a standard Gaussian random
vector. Moreover, by a standard concentration bound on the norm of Gaussian
random vectors (e.g. Corollary 2.3 in \cite{barvinok05}, with
$\epsilon=1/2$):
\[
\max\left\{\Pr\left(\norm{\bn}_2\leq \sqrt{\frac{d}{2}}\right),
\Pr\left(\norm{\bn}_2\geq \sqrt{2d}\right)\right\} \leq \exp\left(-\frac{d}{16}\right).
\]
Finally, for any value of $\bn$, we always have
$\frac{\norm{\bn}_{\infty}}{\norm{\bn}_2}\leq 1$, since the Euclidean norm is
always larger than the infinity norm. Combining these observations, and using
$\mathbf{1}_A$ for the indicator function of the event $A$, we have
\begin{align}
  \E[\norm{\bu}_\infty^4]&=\E\left[\frac{\norm{\bn}_{\infty}^4}{\norm{\bn}_2^4}\right]\notag\\
  &= \Pr\left(\norm{\bn}_2\leq \sqrt{\frac{d}{2}}\right)\E\left[\frac{\norm{\bn}_{\infty}^4}{\norm{\bn}_2^4}~\middle|~ \norm{\bn}_2\leq \sqrt{\frac{d}{2}}\right]
  +\Pr\left(\norm{\bn}_2> \sqrt{\frac{d}{2}}\right)\E\left[\frac{\norm{\bn}_{\infty}^4}{\norm{\bn}_2^4}~\middle|~ \norm{\bn}_2> \sqrt{\frac{d}{2}}\right]\notag\\
  &\leq \exp\left(-\frac{d}{16}\right)*1+\Pr\left(\norm{\bn}_2> \sqrt{\frac{d}{2}}\right)\E\left[\frac{\norm{\bn}_{\infty}^4}{\left(\sqrt{d/2}\right)^4}~\middle|~ \norm{\bn}_2> \sqrt{\frac{d}{2}}\right]\notag\\
  &= \exp\left(-\frac{d}{16}\right)+\left(\frac{2}{d}\right)^2\E\left[\norm{\bn}_{\infty}^4\mathbf{1}_{\norm{\bn}_2>\sqrt{d/2}}\right]\notag\\
  &\leq \exp\left(-\frac{d}{16}\right)+\frac{4}{d^2}\E\left[\norm{\bn}_{\infty}^4\right].\label{eq:uboundd}
\end{align}
Thus, it remains to upper bound $\E\left[\norm{\bn}_{\infty}^4\right]$ where
$\bn$ is a standard Gaussian random variable. Letting $\bn=(n_1,\ldots,n_d)$,
and noting that $n_1,\ldots,n_d$ are independent and identically distributed
standard Gaussian random variables, we have for any scalar $z\geq 1$ that
\begin{align*}
  \Pr(\norm{\bn}_{\infty}\leq z) &= \prod_{i=1}^{n}\Pr(|n_i|\leq z) ~=~
  \left(\Pr(|n_1|\leq z)\right)^d\\
  &= \left(1-\Pr(|n_1|> z)\right)^d ~\stackrel{(1)}{\geq}~ 1-d\Pr(|n_1|> z)\\
  &= 1-2d\Pr(n_1 > z) ~\stackrel{(2)}{\geq}~ 1-d\exp(-z^2/2),
\end{align*}
where $(1)$ is Bernoulli's inequality, and $(2)$ is using a standard tail
bound for a Gaussian random variable. In particular, the above implies that
\[
\Pr\left(\norm{\bn}_{\infty}> z\right)\leq d\exp(-z^2/2).
\]
Therefore, for an arbitrary positive scalar $r\geq 1$,
\begin{align*}
  \E\left[\norm{\bn}_{\infty}^4\right]&= \int_{z=0}^{\infty}\Pr\left(\norm{\bn}_{\infty}^4> z\right)dz\\
  &\leq \int_{z=0}^{r}1 dz+\int_{z=r}^{\infty}\Pr\left(\norm{\bn}_{\infty}> \sqrt[4]{z}\right)dz\\
  &\leq r+\int_{z=r}^{\infty}d\exp\left(-\frac{\sqrt{z}}{2}\right)dz\\
  &= r+4d(2+\sqrt{r})\exp\left(-\frac{\sqrt{r}}{2}\right).
\end{align*}
In particular, plugging $r=4\log^2(d)$ (which is larger than $1$, since we
assume $d>1$), we get $4(2+2\log(d)+\log^2(d))$. Plugging this back into
\eqref{eq:uboundd}, we get that
\[
\E[\norm{\bu}_\infty^4]~\leq~
\exp\left(-\frac{d}{16}\right)+16\frac{2+2\log(d)+\log^2(d)}{d^2},
\]
which can be shown to be at most $c'\left(\frac{\log(d)}{d}\right)^2$ for all
$d>1$, where $c'<150$ is a numerical constant. In particular, this means that
$\sqrt[4]{\E[\norm{\bu}_\infty^4]}\leq \sqrt[4]{c'}\sqrt{\frac{\log(d)}{d}}$
as required.

\end{document}